\documentclass[letterpaper]{IEEEtran}
\usepackage{amsmath,amsfonts}
\usepackage{amssymb}
\usepackage{algorithmic}
\usepackage{array}
\usepackage[caption=false,font=normalsize,labelfont=sf,textfont=sf]{subfig}
\usepackage{textcomp}
\usepackage{stfloats}
\usepackage{url}
\usepackage{verbatim}
\usepackage{graphicx}
\usepackage{nicefrac}
\usepackage{amsthm} 
\usepackage{etoolbox}
\usepackage{enumitem}
\usepackage{color}

\newtheoremstyle{upright}
{3pt}
{3pt}
{\normalfont}
{}
{\bfseries}
{.}
{ }
{}

\theoremstyle{upright}

\newtheorem{theorem}{Theorem}

\newtheorem{corollary}{Corollary}[theorem]
\newtheorem{remark}{Remark}
\newtheorem{assumption}{Assumption}

\def\BibTeX{{\rm B\kern-.05em{\sc i\kern-.025em b}\kern-.08em
    T\kern-.1667em\lower.7ex\hbox{E}\kern-.125emX}}
\usepackage{balance}
\begin{document}
\title{\LARGE Energy-Optimal Spatial Iterative Learning within a Virtual Tube}
\author{Chen Min, Shuli Lv, Pengda Mao, Huixin Cao, Li Hong, and Quan Quan
\thanks{Chen Min, Shuli Lv, Pengda Mao, Huixin Cao, and Li Hong are with the School of Automation Science and Electrical Engineering, Beihang University, Beijing 100191, China.{\tt\small email:min\_chen@buaa.edu.cn,  lvshuli@buaa.edu.cn, maopengda@buaa.edu.cn, caohuixin@buaa.edu.cn, hongli@buaa.edu.cn}. 

Quan Quan (Corresponding author) is with the School of Automation Science and Electrical Engineering, Beihang University, Beijing 100191, China, and with Tianmushan Laboratory, Beihang University, Hangzhou, 311115. {\tt\small email:qq\_buaa@buaa.edu.cn}. 

This work has been submitted to the IEEE for possible publication. Copyright may be transferred without notice, after which this version may no longer be accessible.
}}


\maketitle

\begin{abstract}
    Due to the limited endurance of embedded energy sources such as lithium-polymer (LiPo) batteries, the flight duration and operational range of unmanned aerial vehicles (UAVs) are severely constrained. Although energy-efficient trajectory planning and control have been widely studied, most existing approaches rely on accurate system models and computationally expensive optimization procedures. This paper proposes a model-free online iterative learning (IL) framework to minimize energy consumption. Without requiring explicit models of UAV dynamics or energy consumption, the proposed method improves energy efficiency while maintaining a low computational cost. The per-iteration computational complexity is $O(n)$, where $n$ denotes the number of path points. In the tested cases, the proposed method is approximately 50--60 times faster than the model-based IPOPT benchmark. Simulation results and real-world flight experiments across multiple UAV platforms validate the effectiveness, computational efficiency, and practical applicability of the proposed approach.

\end{abstract}


\section{Introduction}
    \IEEEPARstart{I}{n} recent years, robots have been widely deployed in applications such as delivery, photography, and rescue missions \cite{quan2017introduction}. 
    Their success largely relies on autonomous navigation, namely the ability to reach target locations efficiently using onboard sensing and computation \cite{siegwart2011introduction}. 
    However, the endurance of unmanned aerial vehicles (UAVs) remains fundamentally constrained by onboard energy storage. 
    Most UAVs rely on lithium-polymer (LiPo) batteries, whose limited energy density typically restricts flight time to less than one hour, significantly limiting mission range and duration.
    
    Due to the strict payload and size constraints of UAV platforms, improving energy efficiency has become a critical design objective \cite{7420610,10414185}. 
    Existing approaches to alleviating endurance limitations can be broadly categorized into improving battery technology or enhancing energy utilization efficiency through algorithmic optimization \cite{Karydis2017EnergeticsIR}. 
    While lightweight structures and high-energy-density batteries (e.g., lithium--sulfur batteries \cite{Seh2016}) have been explored, these solutions have yet to achieve substantial breakthroughs or widespread adoption \cite{7487285}. 
    As a result, energy-aware planning and control methods based on existing platforms have attracted increasing attention.
    
    \begin{figure}[htbp]
        \centering
        \includegraphics[scale=0.25]{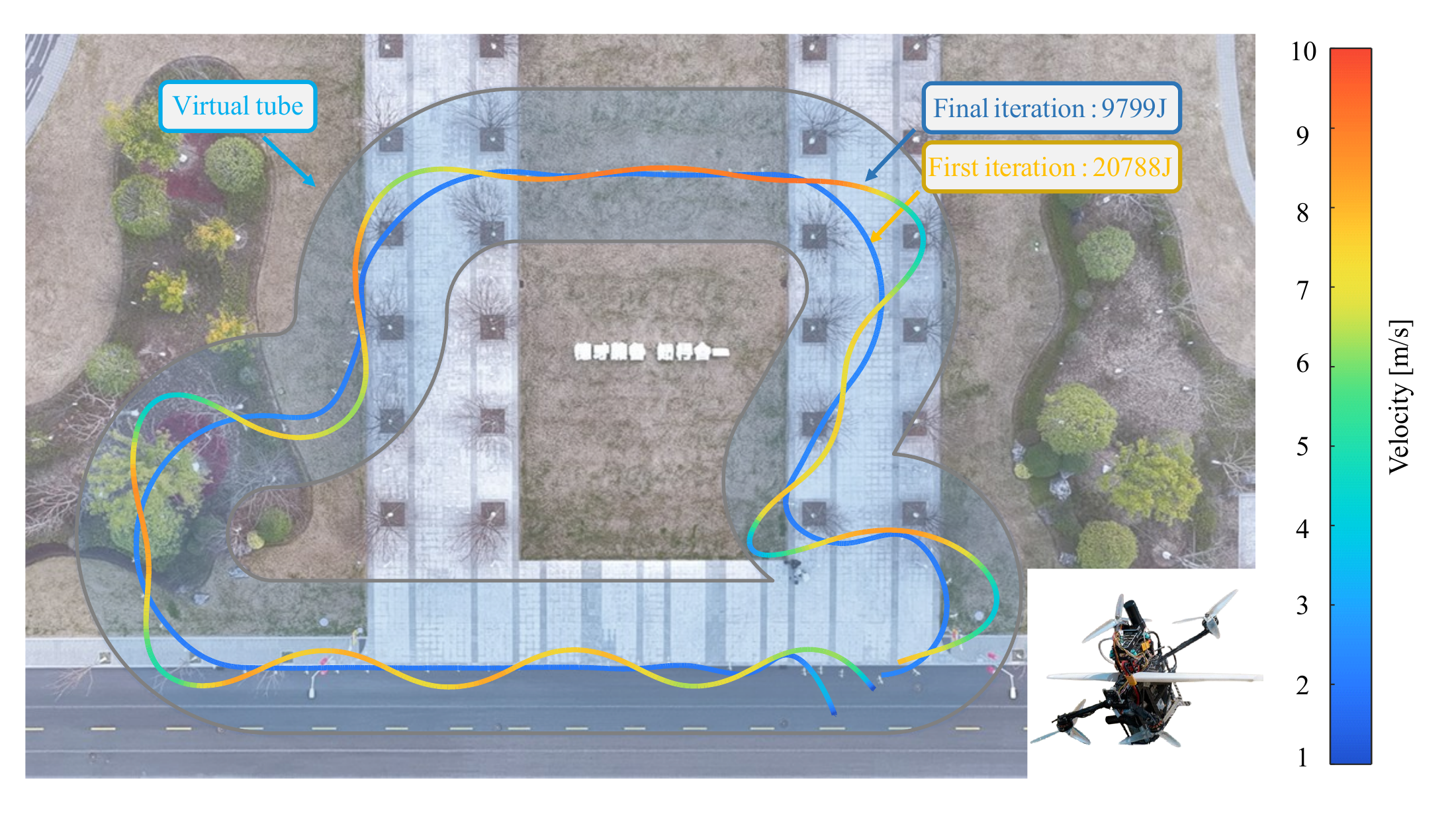}
        \caption{Outdoor experiments under lifting-wing multicopter platform. }
        \label{muticopter}
    \end{figure}
    
    Most existing energy-optimal UAV planning methods rely on explicit energy or dynamic models. 
    One representative class formulates an energy-related objective function and solves an optimal control problem using mathematical optimization techniques \cite{Karydis2017EnergeticsIR,7959079,6868215,8793549,zhai10839300}. 
    Another class constructs empirical energy models for basic flight modes and decomposes complex trajectories into motion primitives \cite{7513397}. 
    Although these methods can incorporate realistic energy characteristics, they are often restricted to predefined path geometries and remain sensitive to modeling inaccuracies. 
    Moreover, there is currently no unified UAV energy modeling standard, and different models may yield substantially different energy estimates for identical missions \cite{zhang2021energy}.
    
    In aviation, the maximum range cruise speed is defined as the velocity that minimizes energy consumption per unit distance traveled \cite{anderson1999aircraft}. While this concept provides a cornerstone for long-endurance missions under steady straight-line flight assumptions, its applicability is limited in robotic UAV applications. Modern UAVs often operate along non-linear trajectories with frequent maneuvers. the energy-optimal speed profile in such cases depends on a complex interplay between velocity, path curvature, and the vehicle's inherent dynamics.

    These challenges are further compounded by the complexity of UAV systems. Accurate modeling of UAV dynamics and energy consumption is demanding and often impractical, limiting model-based optimization \cite{zhang2021energy}. Model-free methods offer a practical compromise, where conventional temporal-domain iterative learning control suits repetitive tasks well \cite{1636313,Wang2009SurveyOI}. However, these time-indexed frameworks focus on tracking accuracy and may suffer from severe time-alignment mismatches when velocity varies dynamically for energy optimization, as in trial-varying path-following tasks \cite{11164953,9761902}.
    
    To overcome this, this paper proposes a model-free energy-optimal spatial iterative learning (IL) method based on power--velocity gradient estimation. By indexing the learning law over the invariant spatial domain rather than time, this approach completely decouples the optimization process from velocity-induced temporal fluctuations \cite{9686620}. A virtual tube \cite{quan2023distributed} enforces spatial constraints, and the spatial IL strategy iteratively adjusts the tangential velocity command along curved trajectories. Without explicit models, the proposed method implicitly and robustly guides the UAV toward an energy-efficient cruising regime.
    \color{black}

    The main contributions of this work are summarized as follows:
    \begin{itemize}
        \item A novel model-free UAV energy optimization method based on power--velocity gradient estimation is proposed, eliminating the dependence on accurate dynamic and energy models.
        \item The proposed algorithm achieves a computational complexity of $O(k^{*}n)$, where $k^{*}$ is the maximum number of iterations and $n$ is the number of waypoints, while exhibiting near-optimal energy performance.
        \item This work extends the application scope of iterative learning control by addressing energy-optimal trajectory tracking through controller-structure innovation.
    \end{itemize}

\section{PRELIMINARY AND PROBLEM FORMULATION\label{sec:Problem-Statement}}

This section presents a simplified UAV control model as the basis for subsequent analysis. A virtually bounded tube is then defined for tracking a static path, followed by a demonstration of the existence of an energy-optimal velocity. Finally, the problem is formulated as an energy-optimal control problem that minimizes energy consumption within the tube.

\subsection{UAV Model}

The full dynamics of a UAV involve aerodynamic effects, attitude coupling, and rotational inertia. To enable tractable analysis and control design, the UAV is modeled as a planar point mass, and only its translational motion and energy-related behavior are considered. The resulting model in $\mathbb{R}^2$ is given by
\begin{equation}
  \left\{ 
  \begin{aligned}
    \dot{\mathbf p}(t) &= \mathbf v(t),\\
    \dot{\mathbf v}(t) &= \mathbf a(t),
  \end{aligned}
  \right.
  \label{eq:point-mass-model}
\end{equation}
where $\mathbf p(t), \mathbf v(t), \mathbf a(t) \in \mathbb{R}^2$ denote the position, velocity, and acceleration, respectively. Instead of explicitly modeling attitude dynamics, the closed-loop translational behavior induced by the inner-loop controller is approximated by a first-order velocity tracking model with respect to a commanded velocity $\mathbf v_{\mathrm c}(t)$ \cite{11164953}:
\begin{equation}    
  \left\{ 
  \begin{aligned}
    \dot{\mathbf p}(t) &= \mathbf v(t),\\
    \dot{\mathbf v}(t) &= -\tau \big( \mathbf v(t) - \mathbf v_{\mathrm c}(t) \big),
  \end{aligned}
  \right.
  \label{eq:point-mass-dynamic}
\end{equation}
where $\tau>0$ is a parameter determined by the UAV's dynamics and inner-loop controller, identified empirically. The commanded velocity is constrained by actuator limits and is saturated as
\begin{equation}
  \mathbf v_{\mathrm c}
  = \mathrm{sat}(\mathbf v_{\mathrm c}', v_{\max})
  \triangleq
  \begin{cases}
    \mathbf v_{\mathrm c}', & \|\mathbf v_{\mathrm c}'\| \le v_{\max},\\[6pt]
    v_{\max} \dfrac{\mathbf v_{\mathrm c}'}{\|\mathbf v_{\mathrm c}'\|}, & \|\mathbf v_{\mathrm c}'\| > v_{\max},
  \end{cases}
  \label{eq:sat-kappa}
\end{equation}
where $\mathbf v_{\mathrm c}'$ denotes the unconstrained command and $v_{\max}$ is the maximum allowable speed.

To analyze motion along a given reference path, the model is reformulated in the spatial domain using the traveled distance $l$ as the independent variable. Let $\mathrm D_l=\frac{\mathrm{d}}{\mathrm{d}l}$ denote the spatial derivative\cite{4668486}, \cite{9525166}. By the chain rule,
\begin{equation}
  \mathrm D_t = v(l)\,\mathrm D_l,
  \label{eq:time-to-space-deriv}
\end{equation}
where $v(l)=\frac{\mathrm{d}l}{\mathrm{d}t}>0$ is the traversal speed. Applying \eqref{eq:time-to-space-deriv} to \eqref{eq:point-mass-dynamic} yields the spatial-domain dynamics
\begin{equation}    
  \left\{ 
  \begin{aligned}
    \mathrm D_l \mathbf p(l) &= \frac{1}{v(l)}\, \mathbf v(l), \\
    \mathrm D_l \mathbf v(l) &= -\frac{\tau}{v(l)} \big( \mathbf v(l) - \mathbf v_{\mathrm c}(l) \big),
  \end{aligned}
  \right.
  \label{eq:spatial-model}
\end{equation}
where $\mathbf v_{\mathrm c}(l)$ inherits the saturation constraint in \eqref{eq:sat-kappa}. Since $v(l)>0$, the mapping $l(t) = \int_0^t v(s)\mathrm{d}s$ is strictly increasing, ensuring the existence of a unique inverse $t=f^{-1}(l)$.

Finally, the velocity tracking error is defined as
\begin{equation}
  \mathbf v(l) = \mathbf v_{\mathrm c}(l) + \delta \mathbf v(l),
  \label{eq:velocity-error}
\end{equation}
which will be used for subsequent analysis in the spatial domain.

\subsection{Virtual Tube for Navigation}
{
    The concept of a virtual tube provides a safe and obstacle-free operational region for robots \cite{9981842}. 
    By constraining the robot within such a region, motion planning and control are simplified while maintaining safety.
    This work builds on our previous studies on virtual-tube-based trajectory planning and pass-through control \cite{9981842, GAO2022107800, quan2023distributed}.

    A virtual tube is a regular curved tube designed on a two-dimensional plane, as shown in Fig.~\ref{tube}. A specific position $\mathbf{p}$ within the virtual tube can be parametrically expressed as:
    \begin{equation}
      \mathcal{T}(l,\theta,\rho)
      = \boldsymbol{\gamma}(l)
      + \rho r_t(l)\mathbf{n}(l)\cos\theta,
      \label{eq:vitural_tube_with_generator}
    \end{equation}
    where $\theta \in \{0, \pi\}$, $l \in [0, L]$, and $\rho \in [0, 1]$. 
    The curve $\boldsymbol{\gamma}(l)$ is the generating curve of the virtual tube. The vector $\mathbf{n}(l)$ represents the unit normal vector of the generating curve, and $r(l)$ is a continuous real-valued function with respect to $l$, which represents the local width (or radius) of the virtual tube. 
    The variable $l$ represents the arc length of the generating curve from the starting point $\boldsymbol{\gamma}(0)$, and $\boldsymbol{\gamma}(l)$ is the position with the arc length $l$ along the generating curve. The constant $L > 0$ represents the total length of the generating curve. The boundary of $\mathcal{T}$ is denoted by $\partial \mathcal{T} $. The generation and construction of virtual tubes have been investigated in existing studies on UAV motion planning \cite{9981842}. In this work, the virtual tube is adopted as a predefined safe boundary for motion planning.
\color{black}

    For any $\mathbf{p}\in\mathcal{T}$, its projection onto the reference path is defined by
    \begin{equation}
      \mathbf{m}(\mathbf{p})
      = \arg\min_{s \in [0,L]}
      \|\gamma(s)-\mathbf{p}\|^2.
      \label{eq:mapping}
    \end{equation}
    The corresponding tangential direction
    $\mathbf{t}_{\mathrm c}=\mathrm{D}_l\,\mathbf{m}(\mathbf{p})$
    defines the motion direction inside the virtual tube, as illustrated in Fig.~\ref{IL decomposition}.
    \begin{figure}[htbp]
        \centering{}\includegraphics[scale=0.55]{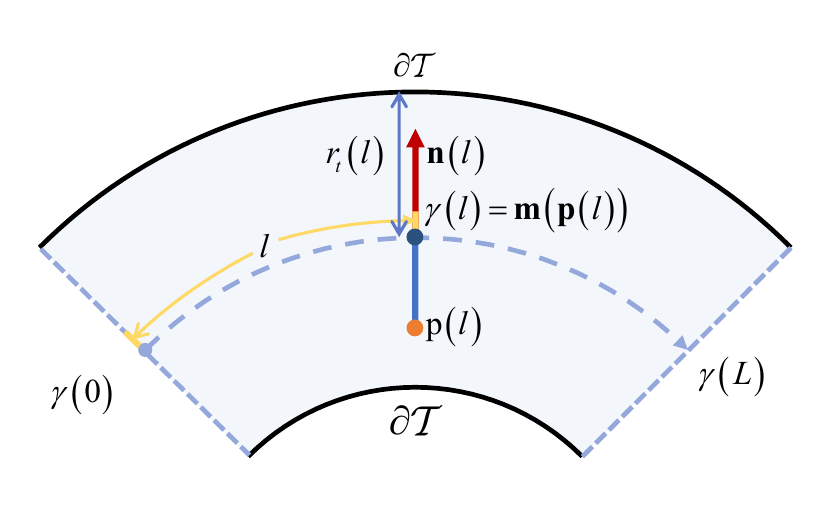}
        \caption{An example of virtual tubes on a two-dimensional plane. }
        \label{tube}
    \end{figure}
}

\subsection{Existence of the Optimal Energy-Consuming Velocity}
    In the time domain, the energy consumption rate of a UAV equals its instantaneous electrical power consumption, namely,
    \begin{equation}
      \mathrm{D}_t E(t) = P(t),
      \label{eq:energy_in_time_domain}
    \end{equation}
    where $E(t)$ denotes the accumulated energy over $[0,t]$ and $P(t)>0$ denotes the instantaneous electrical power consumption. \color{black}The total energy consumption is given by
    \begin{equation}
      E = \int_{0}^{T} P(t)\, \mathrm{d}t,
      \label{eq:energy consumption}
    \end{equation}
    where $T=\int_{0}^{L} \frac{1}{v(l)}\, \mathrm{d}l$ is the total traversal time along a path of fixed length $L$.
    
    Consider a UAV traveling along a fixed reference path $\boldsymbol{\gamma}$ of length $L$. Using the time--space transformation in \eqref{eq:time-to-space-deriv}, the total energy consumption can be expressed in the spatial domain as
    \begin{equation}
        E = \int_{0}^{L} \frac{P(v(l))}{v(l)} \, \mathrm{d}l = \int_{0}^{L} P(l,v) \, \mathrm{d}l,
        \label{eq:energy_in_spatial_domain}
    \end{equation}
    where $v(l)={\mathrm{d}l}/{\mathrm{d}t}>0$ denotes the translational speed along the path, and $P(l,v) \triangleq {P(v(l))}/{v(l)}$ explicitly characterizes the spatial energy consumption rate depending on the local speed.
    
    To analyze the existence of an energy-optimal velocity, we partition the path into sufficiently small spatial segments.
    Within a small segment $[l_i, l_i+\Delta l]$, the velocity variation is negligible and can be approximated as a constant $v_i$.
    The energy consumed in this segment can then be approximated as
    \begin{equation}
        E_i(v_i) \approx \Delta l \cdot P(l_i, v_i).
    \end{equation}
    Thus, minimizing the local energy consumption is equivalent to minimizing $P(l_i, v_i)$ over the admissible speed interval $[v_{\min}, v_{\max}]$.
    
    In practical UAV propulsion systems, the instantaneous power $P(v)$ primarily comprises a constant baseline power and an aerodynamic drag term that increases smoothly with $v$. This ensures $P(v)$ is continuous and strictly positive on any compact interval $[v_{\min}, v_{\max}]$ with $0<v_{\min}<v_{\max}<\infty$.
    Consequently, both $P(v)$ and $1/v$ are continuous on this interval, rendering the spatial rate $P(l,v)$ continuous as well. By the Weierstrass extreme value theorem \cite{martinez2014weierstrass}, $P(l,v)$ attains a global minimum within $[v_{\min}, v_{\max}]$.
    
    \subsection{Problem Formulation}
    
    The objective of this study is to enable a UAV to pass through a virtual tube safely while minimizing total energy consumption.
    Unlike conventional trajectory-tracking problems, the focus here is on energy-optimal motion generation under safety and dynamic constraints.
    
    Furthermore, because aerodynamic drag typically manifests as a polynomial function of velocity (often involving cubic terms), the power model is inherently smooth without non-differentiable points within the physical flight envelope. Based on the analysis in \textit{Subsection II-C}, we make the following assumption:
    
    \begin{assumption}\label{assum:smooth_power}
        The spatial energy consumption rate $P(l,v)$ is twice continuously differentiable with respect to $v$, i.e., $P(l,v)\in \mathcal{C}^2$ in $v$.
    \end{assumption}
    
    The energy-optimal navigation problem is then formulated as
    \begin{equation}
        \begin{aligned}
            \min_{\mathbf{v}_{\mathrm{c}}(l)\in\mathbb{R}^2} \quad
            & E \\
            \text{s.t.} \quad
            & \eqref{eq:spatial-model}, \\
            & \|\mathbf{v}_{\mathrm{c}}(l)\| \le v_{\max}, \\
            & \mathbf{p}(l) \in \mathcal{T}.
        \end{aligned}
        \label{eq:Problem_Formulation}
    \end{equation}
    
    The goal is to solve \eqref{eq:Problem_Formulation} using a spatial IL framework, which fundamentally differs from classical IL approaches that primarily focus on trajectory tracking.

    \section{ENERGY-OPTIMAL SPATIAL IL\label{sec:Erengy optimal}}
    This section first presents the design rationale and specific formulation of the control law. Based on this foundation, the energy-optimal spatial IL method tailored for the virtual tube is developed, accompanied by its convergence proof. Finally, the algorithmic complexity of the spatial IL is analyzed.

\subsection{Control Law Design}
{
	The energy-optimal spatial IL iteratively regulates velocity to achieve an optimal trade-off between energy consumption reduction and path tracking accuracy, thereby deriving a velocity allocation strategy based on spatial coordinates.
	
	As shown in Fig.\ref{IL decomposition}, based on the initial path $\gamma$  and the virtual tube obtained from the path planner, the corresponding IL structure is designed as follows:
	\begin{equation}
		\begin{aligned}
			\mathbf{v}'_{\mathrm{c},k}(l) = \mathbf{v}_{\mathrm{h},k}(l)+\mathbf{v}_{\mathrm{p},k}(l),
		\end{aligned}
		\label{eq:IL structure}
	\end{equation}
	where $\mathbf{v}'_{\mathrm{c},k}(l)$ denotes the original control command in the $k$-th iteration, with $0 < k \leq k^*$ and $k\in \mathcal{N}^+$ ($k^*$ is the maximum iteration number). In this command, $\mathbf{v}_{\mathrm{h},k}(l)$ is used for the tangential component and $\mathbf{v}_{\mathrm{p},k}(l)$ is used for path convergence, namely helping the UAV track the path $\gamma$.
}
\subsection{Convergence Control}
{
    
Since the UAV has a closest point on the initial path, the error vector of the UAV from the initial path is defined as
\begin{equation}
	\begin{aligned}
		\mathbf{e}_{\mathrm{p},k}(l) = \mathbf{m}(\mathbf{p}_k(l))-\mathbf{p}_k(l).
	\end{aligned}
	\label{eq:error}
\end{equation}
\color{black}
    In addition, the size of the error is defined as
    	\begin{equation}
		\begin{aligned}
			e_{\mathrm{p},k}(l) =  \| \mathbf{e}_{\mathrm{p},k}(l) \|.
		\end{aligned}
		\label{eq:error_size}
	\end{equation}

    \begin{figure}[htbp]
		\centering{}\includegraphics[scale=0.8]{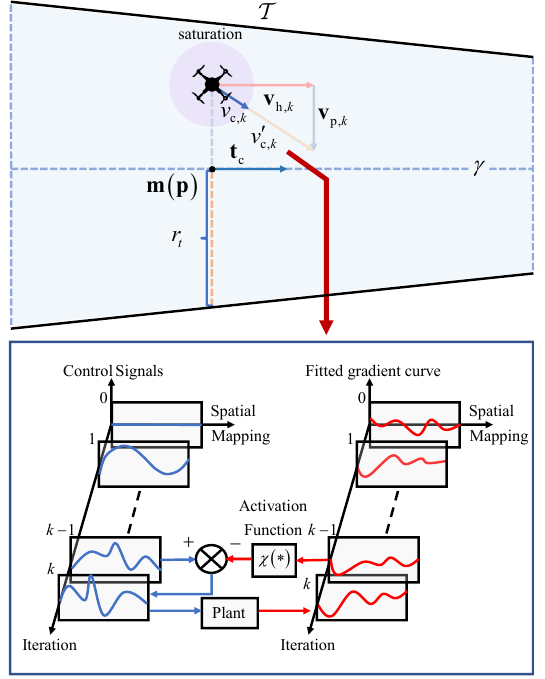}
		\caption{IL decomposition.}
		\label{IL decomposition}
	\end{figure}

	The designed convergence control term is error-based to ensure that the robot follows the tracking path. To achieve smooth and swift lateral convergence, a constant proportional-derivative control law is deployed in this paper as
    \begin{equation}
    	\begin{aligned}
    		\mathbf{v}_{\mathrm{p},k}(l) = \left( k_\mathrm{p} \mathbf{e}_{\mathrm{p},k}(l) + k_\mathrm{d} \mathrm D_le_{\mathrm{p},k}(l)\frac{{\mathbf{e}}_{\mathrm{p},k}(l)}{e_{\mathrm{p},k}(l)} \right) v^*_{k}(l),
    	\end{aligned}
    	\label{eq:convergence control}
    \end{equation}
    where $ v^*_{k}(l) $ denotes the velocity adjustment term for ILC, which will be elaborated later; $e_{\mathrm{p},k}(l)$ is the path tracking error, and $\mathrm D_le_{\mathrm{p},k}(l)$ represents its spatial derivative. The control parameters $k_\mathrm{p} > 0$ and $k_\mathrm{d} > 0$ are preset constant tracking gains.
    
	\begin{remark}
    To ensure the UAV remains within the virtual tube, a control input based on the Barrier Lyapunov Function (BLF) can also be designed as:
    \begin{align}
    \mathbf{v}_{\mathrm{p},k}(l) = -k_1(l)\frac{\mathbf{e}_{\mathrm{p},k}(l)}{r_t^2(l) -e^2_{\mathrm{p},k}(l)},
    \label{eq:BLF}
    \end{align}
    where $ k_1(l) > 0$ determines the corrective intensity of the BLF-based feedback term. 
    Both \eqref{eq:convergence control} and \eqref{eq:BLF} generate corrective velocities for reducing the path tracking error, with different gain structures. The former uses preset constant gains, while the latter uses a state-dependent gain near the tube boundary. In this paper, \eqref{eq:convergence control} is used in the simulations and experiments.
    \end{remark}
    \color{black}
}

\subsection{Model-free Energy Learning}
{
	To adjust the tangential traversal pace, the pace controller is designed as
	\begin{equation}
		\begin{aligned}
			\mathbf{v}_{\mathrm{h},k}(l) = v^*_{k}(l)\mathbf{t_c},
		\end{aligned}
		\label{eq:pace controller}
	\end{equation}
	where $v^*_{k}(l)$ is consistent with that in \eqref{eq:convergence control},
	and $\mathbf{t_c}$ indicates the unit tangent direction at the location $l$.

	As demonstrated in $ \textit {Section II}$, the optimization energy objective for the planned path is transformed from the time domain to the spatial domain, 
	converting the original variable-upper-limit integral problem into a fixed-integral formulation. Consequently, the extremum condition of the objective function can be derived based on the core Euler-Lagrange equation of the variational method \cite{berberian1999introduction}, namely,

	\begin{equation}
		\begin{aligned}
			\frac{\mathrm{d}}{\mathrm{d}l}\left(\frac{\partial (P(l,v))}{\partial (\mathrm{D}_l v)}\right) - \frac{\partial (P(l,v))}{\partial v} = 0.
		\end{aligned}
		\label{eq:Euler-Lagrange Equation-1}
	\end{equation}
    
    For cruising-dominated trajectories with moderate velocity variations, acceleration-induced energy is treated as secondary, and the spatial energy rate is approximated by $P(l,v)$ rather than $P(l,v,\mathrm{D}_l v)$. Thus,
    \begin{equation}
    \begin{aligned}
    \frac{\partial (P(l,v))}{\partial (\mathrm{D}_l v)} = 0,
    \end{aligned}
    \label{eq:Euler-Lagrange Equation-special}
    \end{equation}
    which yields the pointwise optimality condition
    \begin{equation}
    \begin{aligned}
    \frac{\partial (P(l,v))}{\partial v} = 0.
    \end{aligned}
    \label{eq:Euler-Lagrange Equation-2}
    \end{equation}
    For aggressive maneuvers, a higher-order model including $\mathrm{D}_l v$ would be required.
    \color{black}

	During the flight, variations in velocity and energy are observed in real time, enabling the estimation of the instantaneous power derivative with respect to velocity.
	Based on the above ideas, the learning law of $v^*_{k}$ was designed in the form of spatial IL:
	\begin{equation}
		\begin{aligned}
			{v}^*_{k+1}(l) = {v}^*_{k}\left(l\right) - \chi \left({\partial P(l,v)}/{\partial v}\big|_{v=v_k}\right),
		\end{aligned}
		\label{eq:learning law}
	\end{equation}
	especially used for adjusting the contribution of ${v}_{k}(l)$ in \eqref{eq:IL structure}, $ k $ indicates the iteration
	number and $\chi$ is a function which satisfies
	\begin{equation}
		\begin{aligned}
			\chi(x)=k_{\chi}(x)x,0<\alpha \leq k_{\chi}(x) \leq \beta .
		\end{aligned}
		\label{eq:nonlinear activation function}
	\end{equation}
    
	Following the IL principle, the framework estimates the energy-consumption-rate gradient from experimental data to iteratively optimize the velocity command within the virtual tube. As illustrated in Fig.~\ref{func}, if $\left.\partial P(l,v)/\partial v\right|*{v=v_k} < 0$, the velocity command is increased to reduce energy consumption; otherwise, it is decreased. The nonlinear mapping limits the correction amount of $v^*_{k}(l)$ between two adjacent iterations, preventing abrupt speed variations and improving convergence.

    \begin{remark}
    A bounded activation function $\chi(x)$ can be constructed using a scaled sigmoid function, e.g., $\chi(x)=a\left({1}/{(1+e^{-bx})}-{1}/{2}\right).$
    More generally, any bounded Lipschitz continuous function satisfies the design requirements.
    For analytical convenience, $\chi(x)$ is assumed to be linear with constant slope $k_\chi$ in the subsequent analysis.
    \end{remark}
    
    \begin{remark}
    The proposed method updates the control command using directly measured data, including instantaneous power, velocity, and position information.
    It does not require explicit system dynamics, model parameters (e.g., $\tau$ in \eqref{eq:spatial-model}), or the energy model.
    The parameter $\tau$ is used only for analysis and simulation validation, rather than for the practical learning update.
    Model uncertainties are handled implicitly through the iterative update process.
    Therefore, the method is data-driven and model-free in implementation.
    \end{remark}

    \begin{figure}[htbp]
        \centering
        \includegraphics[scale=0.35]{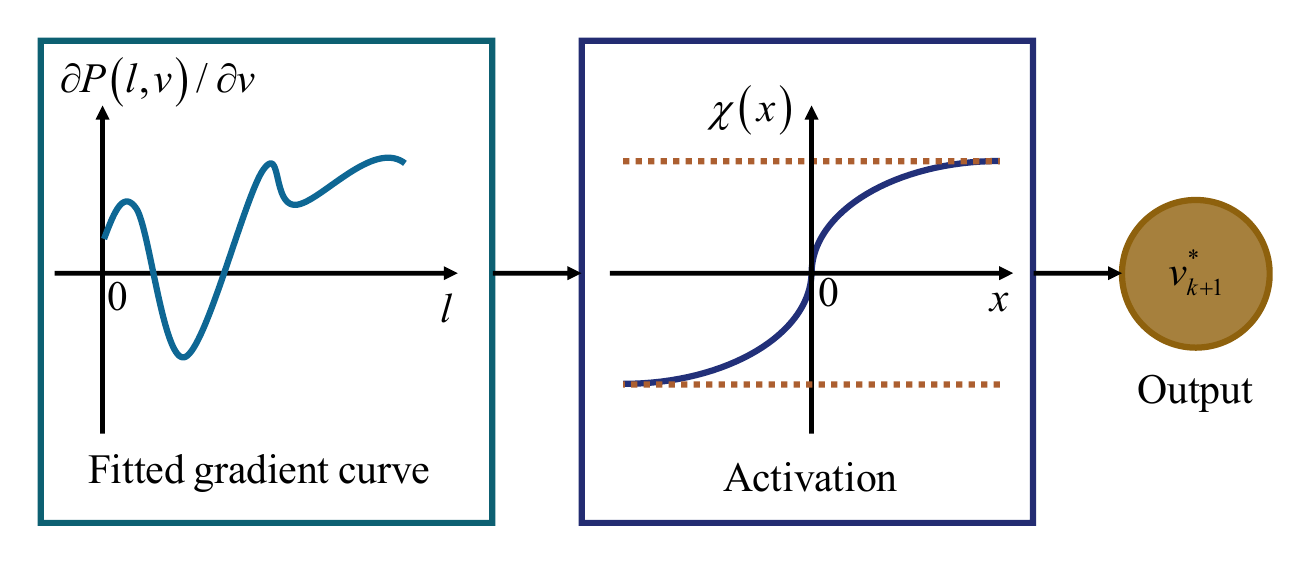}
        \caption{The fitted gradient curve of the energy consumption rate, after being processed by the activation function $\chi$, is applied to update the velocity command $v^*_{k+1}$.}
        \label{func}
    \end{figure}

    \begin{theorem}
        Suppose the UAV model satisfies \eqref{eq:spatial-model}, the controller is designed as \eqref{eq:IL structure}, \eqref{eq:convergence control}, and \eqref{eq:pace controller}, and $M$ denotes an upper bound of $\partial^2 P(l,v)/\partial v^2$.
        With the IL law \eqref{eq:learning law}, if $\delta \mathbf{v} \equiv \mathbf{0}$ and $k_{\chi}$  satisfies $k_{\chi}< 2/M$, then the term $v_k^*(l)$ is uniformly ultimately bounded as $k \to \infty$ and there exists $ \lim_{k \to \infty} E = E^*$.
    \label{convergetheorem}
    \end{theorem}

	\begin{proof}
	\setlength{\parindent}{0pt} See \textit{Appendix A}.
	\end{proof}
	\begin{corollary}
    	Given the UAV model, controller, and learning law in \textit{Theorem 1}, if velocity tracking error satisfies $\max_{l \in [0,L]} \mu=(\delta v_{k+1}(l)-\delta v_k(l))/(v_{k+1}(l)-v_k(l)) < 1$ and $k_{\chi} < 2(1-\mu)/M$, then the term $v_k^*(l)$ is uniformly ultimately bounded as $k \to \infty$ and there exists $ \lim_{k \to \infty} E = E^*$.
	\end{corollary}
	
	\begin{proof}
	\setlength{\parindent}{0pt} See \textit{Appendix B}.
	\end{proof}

}

\subsection{Time Complexity Analysis}
{
    Computational efficiency is determined by the complexity of the learning update. From \eqref{eq:learning law}, the proposed spatial IL update is performed pointwise along the reference path, yielding a per-iteration complexity of $O(n)$, where $n$ is the number of discrete path points. With a predefined maximum iteration number $k^*$, the total complexity is $O(k^*n)$.
    
    By contrast, standard interior-point methods for constrained nonlinear optimization typically have a complexity of $O(n^3)$ with respect to the decision-variable dimension \cite{boyd2004convex}. As the path discretization and constraints increase, this cubic scaling may hinder onboard real-time implementation, even though sparsity exploitation and parallel computation can reduce the practical burden \cite{boyd2004convex,POTRA2000281,nocedal2006numerical}.
    
    \begin{remark}
    For quadrotors and lifting-wing multicopters, full model-based optimization involves high-dimensional dynamics and constraints. The proposed method avoids solving such nonlinear models by using measured power-gradient information to update only the velocity adjustment term $v^*$.
    \end{remark}
}

\section{ SIMULATIONS AND EXPERIMENTS\label{sec: SIMULATIONS AND EXPERIMENTS}}
{
	This section presents simulations and experiments to evaluate the performance and generalization capability of the proposed algorithm using directly observable energy and velocity data.
    In all simulations and experiments, the velocity adjustment term is updated once after each completed trajectory execution, and each update generates $v^*$ over $n$ discrete path points.
    The measured time-domain velocity and power data are projected onto the fixed spatial grid by nearest-path-point registration, so different lap durations remain aligned by the same spatial index $l$. The number of spatial path points is selected according to the total path length and flight speed to ensure appropriate spatial resolution for gradient estimation. The learning process is terminated when the predefined maximum number of trajectory executions is reached, and practical convergence is identified when the final energy values remain within a stable envelope.
    
    The results show that the proposed spatial IL method is model-free, computationally efficient, and applicable across different UAV platforms.
    
    \subsection{Numerical Comparison}
    {
    In this subsection, the simulation utilizes a simplified power model $P(t) = k_5 + k_6\|\mathbf{v}(t)\|^3$ \cite{zhang2021energy}, where $k_5 = 50\,\mathrm{W}$ and $k_6 = 0.05\,\mathrm{kg}\cdot\mathrm{m}^{-1}$. The nonlinear mapping $\chi(x) = 1/(1 + e^{-10x}) - 1/2$ is bounded within $[-0.5, 0.5]$, with $v_{\text{max}}=20.0\,\mathrm{m/s}$ and $v_1^*(l)=4.0\,\mathrm{m/s}$ . This nonzero admissible initial speed is used to prevent calculation errors in $P(l,v)$.\color{black}

    Crucially, both the parameter $\tau$ and the energy model remain unknown to the proposed IL method. To evaluate performance, the method is compared with two representative optimization baselines on the virtual-tube trajectory in Fig.~\ref{iteration}(a). The first baseline is a model-free Sequential Quadratic Programming (SQP) method implemented using fmincon, where the velocity commands at the discrete path points are directly optimized through black-box numerical simulations. This baseline provides a comparison with conventional model-free nonlinear optimization. The second baseline is a model-based optimization framework formulated in CasADi~\cite{andersson2019casadi} and solved by the Interior Point OPTimizer (IPOPT)~\cite{wachter2006ipopt}. This baseline exploits the explicit simplified model, sparsity, and algorithmic differentiation, and therefore serves as a more competitive modern numerical optimization benchmark.
    
     Results in Fig.~\ref{iteration} and Table~1 confirm that the IL method achieves energy efficiency comparable to both benchmarks, with rapid convergence of trajectory energy. The velocity adjustment term $v^*(l)$ in Fig.~\ref{iteration}(c) evolves progressively over the discrete path points and becomes stable as the iteration proceeds, demonstrating convergence of the learned velocity profile. Notably, the SQP baseline is formulated as a black-box nonlinear program that optimizes the velocity commands over 10,000 spatial path points under velocity and virtual-tube constraints; under this setting, the model-free SQP requires approximately 4.1 hours, and the model-based IPOPT takes 30 to 40 seconds. In contrast, the proposed spatial IL method converges in approximately 0.6 to 0.7~s. This efficiency gain empirically supports the $O(n)$ complexity of our framework, demonstrating that the proposed data-driven approach reduces the computational burden while preserving competitive energy performance.
    \color{black}
    	\begin{figure}[htbp]
    		\centering
    		\includegraphics[scale=0.38]{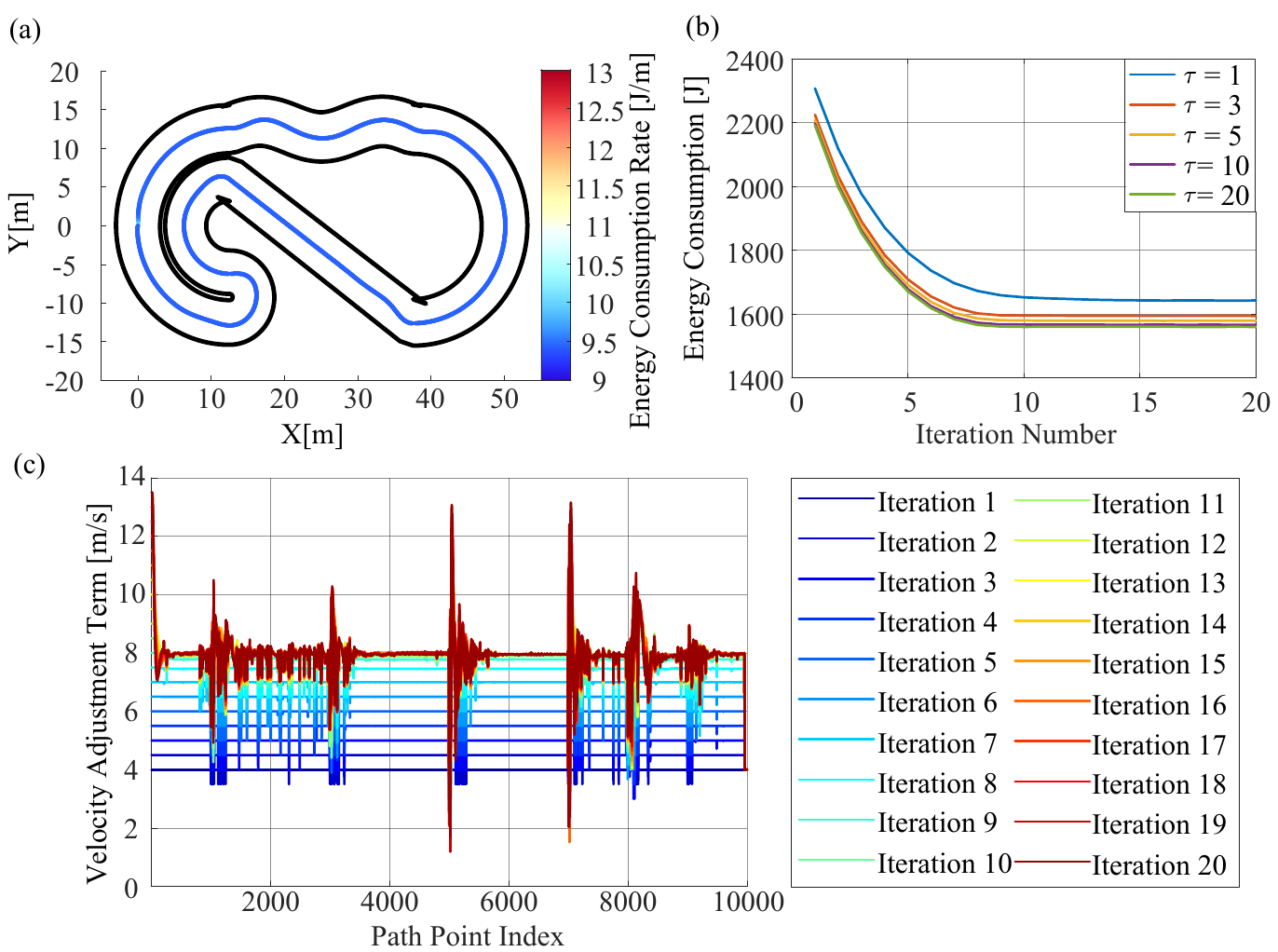}
    		\caption{ (a) Distribution of position and energy consumption rate for the IL method inside the tube using $\tau$ = 5. (b) The energy consumption varied with training iterations. (c) Evolution of the velocity adjustment term $ v_k^*(l)$ over iterations.}
    		\label{iteration}
    	\end{figure}

    \begin{table}[htbp]
    \centering
    \caption{Energy Consumption and Optimization Time of Different Algorithms}
    \label{tab2}	
    \begin{tabular}{m{0.01\textwidth}<{\centering}|m{0.05\textwidth}<{\centering}|m{0.05\textwidth}<{\centering}|m{0.06\textwidth}<{\centering}|m{0.05\textwidth}<{\centering}|m{0.035\textwidth}<{\centering}|m{0.06\textwidth}<{\centering}}
        \hline
        & \multicolumn{3}{c|}{Energy consumption (J)} & \multicolumn{3}{c}{Optimization Time (s)} \\ 
        \hline
        & \multicolumn{2}{c|}{Model-free} & \makebox[0.06\textwidth][c]{Model-based} & \multicolumn{2}{c|}{Model-free} & \makebox[0.045\textwidth][c]{Model-based} \\
        \hline
        $\tau $ & SQP &\textbf{SIL} & IPOPT  & SQP &\textbf{SIL} & IPOPT  \\
        \hline
        1 & 1650.59 & \textbf{1643.00} & 1556.50  & 9139.04 &\textbf{0.69} & 39.82 \\
        \hline
        3 & 1592.64 & \textbf{1594.44} & 1555.55  & 4812.23 &\textbf{0.65} & 35.41 \\
        \hline
        5 & 1574.58 & \textbf{1574.79} & 1555.36  & 5504.69 &\textbf{0.68} & 35.61  \\
        \hline
        10 & 1566.22 & \textbf{1566.41} & 1555.22  & 6597.91 &\textbf{0.61} & 40.16  \\ 
        \hline
        20 & 1558.46 & \textbf{1560.14} & 1555.15  & 5166.80 &\textbf{0.62} & 35.95  \\ 
        \hline
    \end{tabular}
\end{table}

    }
	\subsection{Simulation Flight Experiment}
    {
        
    	\begin{figure*}
    		\centering{}\includegraphics[scale=0.32]{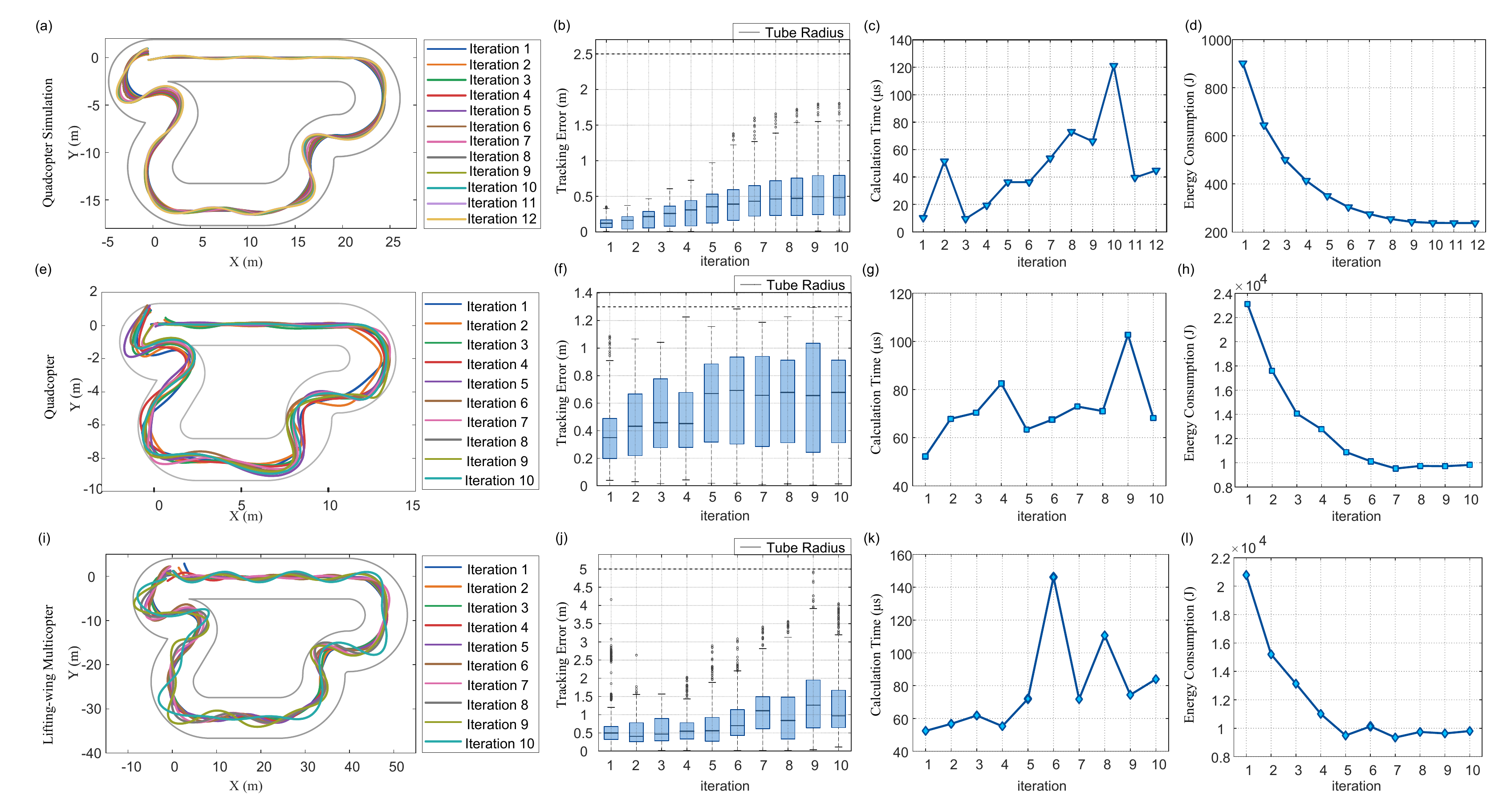}
    		\caption{Performance of the proposed algorithm in simulation and real flight. (a) Flight trajectory in simulation.
        (b) Tracking-error distribution in simulation.
        (c) Computation time for calculating the next-iteration velocity command in simulation.
        (d) Energy consumption for a single simulated flight.
        (e) Flight trajectory of the quadcopter.
        (f) Tracking-error distribution of the quadcopter.
        (g) Computation time for calculating the next-iteration velocity command of the quadcopter.
        (h) Energy consumption for a single quadcopter flight.
        (i) Flight trajectory of the lifting-wing multicopter.
        (j) Tracking-error distribution of the lifting-wing multicopter.
        (k) Computation time for calculating the next-iteration velocity command of the lifting-wing multicopter.
        (l) Energy consumption for a single lifting-wing multicopter flight.}
    		\label{sim_and_exp_road}
    	\end{figure*}
        
        \begin{figure}[htbp]
    		\centering
    		\includegraphics[scale=0.4]{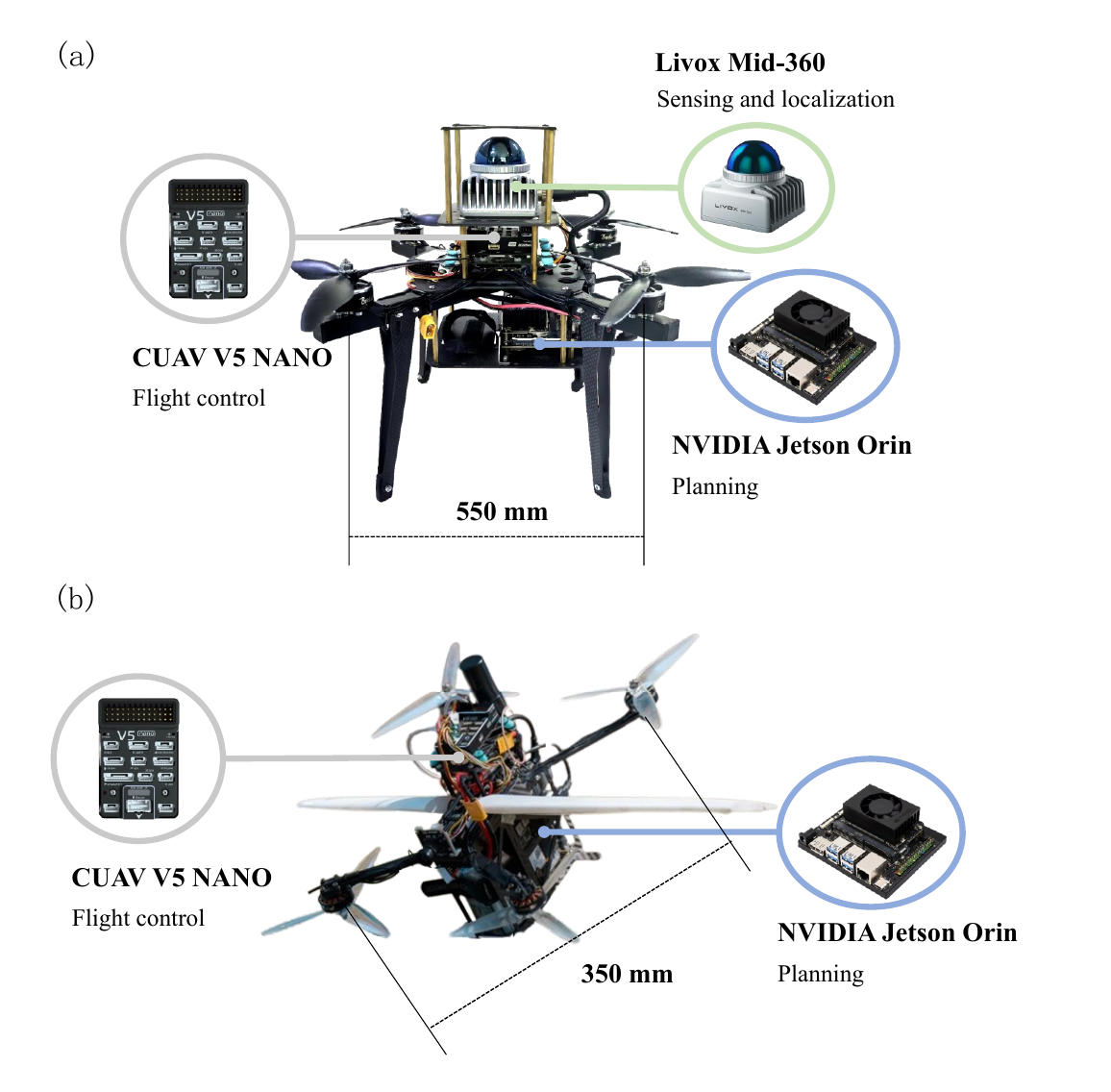}
    		\caption{ The different experimental platforms employed to verify the generalization capability of the proposed algorithm. (a) A quadcopter platform. (b) A lifting-wing multicopter platform.}
    		\label{platform}
    	\end{figure}
    	To validate energy-optimal IL under unknown UAV dynamics and energy models, software-in-the-loop (SITL) experiments are conducted using PX4 (v1.10) with Gazebo. The activation function in Eq.~\eqref{eq:nonlinear activation function} is set to $\chi(x)=1/(1+e^{-10x})-1/2$. 
    	The maximum iteration number is $k^*=12$. In the first flight, the waypoint velocity command is $1.0\,\mathrm{m/s}$, starting from zero velocity, with $v_{\text{max}}=10.0\,\mathrm{m/s}$.
    
    	Results are shown in Fig.~\ref{sim_and_exp_road}(a)--(d). 
    	The computation time for updating the next-flight velocity command ranges from $0.1\times10^{-4}$ to $1.2\times10^{-4}\,\mathrm{s}$. 
    	The energy consumption decreases from $901.85\,\mathrm{J}$ and converges to approximately $238\,\mathrm{J}$ after nine iterations. These results verify the theoretical feasibility of the proposed IL framework in achieving significant energy reduction (approx. 73.6\%) while maintaining high computational efficiency.
    }
	\subsection{Real-world Flight Experiment of a Quadcopter}
    {
    	To evaluate practical performance, outdoor autonomous flight experiments are conducted using the quadrotor platform shown in Fig.~\ref{platform}(a), equipped with an NVIDIA Jetson Orin and a Mid-360 LiDAR. 
    	Mapping and localization are achieved via Fast-LIO. 
    	Experiments are performed in a $30\,\mathrm{m}\times20\,\mathrm{m}\times5\,\mathrm{m}$ workspace with all computations executed onboard in real time. The activation function remains $\chi(x)=1/(1+e^{-10x})-1/2$, and the maximum iteration number is 10. The initial waypoint velocity is $1.0\,\mathrm{m/s}$, and $v_{\text{max}}=4.0\,\mathrm{m/s}$. The next-flight command is updated solely from measured current, voltage, velocity, and position, without prior knowledge of UAV dynamics or energy models.
    
    	As shown in Fig.~\ref{sim_and_exp_road}(e)--(h), the onboard computation time ranges from $0.5\times10^{-4}$ to $1.2\times10^{-4}\,\mathrm{s}$. 
    	The energy consumption decreases from $23122.37\,\mathrm{J}$ and converges to approximately $9700\,\mathrm{J}$ after seven iterations. The rapid convergence in real-world environments indicates that the algorithm is robust against external disturbances and measurement noise. Furthermore, the microsecond-level computation time ensures its suitability for real-time onboard optimization on resource-constrained platforms.

    }
    \subsection{Real-world Flight Experiment of a Lifting-wing Multicopter}
    {
    To further demonstrate generalization, autonomous experiments are conducted on the lifting-wing multicopter shown in Fig.~\ref{platform}(b). Unlike conventional quadcopters
    , the lifting-wing multicopter can use its wing to share part of the lift during forward flight, while the aerodynamic drag also changes with velocity. Therefore, this platform provides a useful test case for evaluating whether the proposed algorithm can optimize energy consumption without explicitly modeling the lift-drag characteristics. The UAV is equipped with an NVIDIA Jetson Orin, with positioning achieved via GPS and a laser rangefinder. Flights are performed in a $70\,\mathrm{m}\times50\,\mathrm{m}\times30\,\mathrm{m}$ outdoor workspace with full onboard real-time computation. The activation function is set to $\chi(x)=4(1/(1+e^{-10x})-1/2)$, and the maximum iteration number is 10. The initial waypoint velocity is $2.0\,\mathrm{m/s}$ and $v_{\text{max}}=10.0\,\mathrm{m/s}$.
    
    As shown in Fig.~\ref{sim_and_exp_road}(i)--(l), the computation time ranges from $0.5\times10^{-4}$ to $1.5\times10^{-4}\mathrm{s}$. The energy consumption decreases from $20787.86\,\mathrm{J}$ to approximately $9700\,\mathrm{J}$ after seven iterations. The learned velocity profile reflects the platform-dependent power-speed characteristics from measured data. Although the complete lift-drag transition is limited by the safety velocity bound, the results from the lifting-wing multicopter suggest the model-free adaptability of the proposed energy-optimal IL framework.
    }

}

\section{ CONCLUSION\label{sec: CONCLUSION}}
{
    This paper presents a model-free, energy-optimal spatial IL framework for UAV trajectory tracking without explicit energy or dynamic models. The main conclusions are:
    1) Significant energy optimization is achieved under real-world constraints, validating the model-free spatial law.
    2) The computational cost is only 1.52\%--1.92\% of IPOPT, enabling resource-constrained embedded deployment.
    3) Rapid convergence across quadcopter and lifting-wing platforms demonstrates feasibility.
    
    Operational limitations include intentionally trading tracking performance for global systemic energy efficiency within the virtual tube, and requiring actual flight iterations to optimize brand-new trajectories. The current framework is simple, but additional control objectives (e.g., stricter centerline tracking or obstacle avoidance) can be incorporated. Future work will focus on eliminating these limitations and extending the framework to multi-objective or tighter-path constraints.
    \color{black}
}
\appendices
\section{Proof of Theorem 1}
    The Taylor expansion of the performance function $P(l,v)$ at the iteration point $v_k$ is given by
    \begin{equation}
    	\begin{aligned}
    		&P(l,v_{k+1}) = P(l,v_k) + {\partial P(l,v)}/{\partial v}\big|_{v=v_k}(v_{k+1}-v_k)\\ &+\frac{1}{2}{\partial^2P(l,v)}/{\partial v^2}\big|_{v=v_k}(v_{k+1}-v_k)^2+\cdots.
    	\end{aligned}
    	\label{eq:proofA_1}
    \end{equation}
    Applying the mean value theorem, \eqref{eq:proofA_1} can be rewritten as
    \begin{equation}
    	\begin{aligned}
    		&P(l,v_{k+1}) = P(l,v_k) + {\partial P(l,v)}/{\partial v}\big|_{v=v_k}(v_{k+1}-v_k)\\ &+\frac{1}{2}{\partial^2P(l,v)}/{\partial v^2}\big|_{v=\xi }(v_{k+1}-v_k)^2,
    	\end{aligned}
    	\label{eq:proofA_2}
    \end{equation}
    where $\xi \in [v_k, v_{k+1}]$ (assuming $v_k < v_{k+1}$).
    
    According to $\textit{Assumption 1}$, the second-order derivative
    $\partial^2 P(l,v)/\partial v^2$
    is continuous on the compact set $[v_{\min},v_{\max}]$.
    By the Weierstrass extreme value theorem \cite{martinez2014weierstrass},
    there exists a constant $M>0$ such that
    \begin{equation}
    	\begin{aligned}
    		\left\| {\partial^2P(l,v)}/{\partial v^2} \right\Vert \leq M.
    	\end{aligned}
    	\label{eq:proofA_second derivative}
    \end{equation}
    Combining \eqref{eq:learning law}, \eqref{eq:nonlinear activation function}, \eqref{eq:proofA_2} and \eqref{eq:proofA_second derivative}, we obtain
    \begin{equation}
    	\begin{aligned}
    		&P(l,v_{k+1}) \leq P(l,v_k) + {\partial P(l,v)}/{\partial v}\big|_{v=v_k}(v_{k+1}-v_k)\\ &+{M}/{2}(v_{k+1}-v_k)^2.
    	\end{aligned}
    	\label{eq:proofA_3}
    \end{equation}
    Substituting the learning law into \eqref{eq:proofA_3} yields
    \begin{equation}
    	\begin{aligned}
    		&P(l,v_{k+1}) \leq P(l,v_k) +\left({Mk_{\chi}^2}/{2}-k_{\chi}\right) \left({\partial P(l,v)}/{\partial v}\big|_{v=v_k}\right)^2.\\
    	\end{aligned}
    	\label{eq:proofA_4}
    \end{equation}
    When $k_{\chi}<2/M$, the coefficient in \eqref{eq:proofA_4} is strictly negative, i.e.,
    \begin{equation}
    	\begin{aligned}
    		\left({Mk_{\chi}^2}/{2}-k_{\chi}\right)\left({\partial P(l,v)}/{\partial v}\big|_{v=v_k}\right)^2 \leq 0.
    	\end{aligned}
    	\label{eq:proofA_5}
    \end{equation}
    Therefore, $P(l,v)$ decreases monotonically with respect to the iteration index $k$.
    Since $P(l,v)$ is bounded below on a compact domain, it necessarily converges.
    Thus, as $k\to\infty$, there exist
    \begin{equation}
    	\left\{
    	\begin{aligned}
    		&\lim_{k \to \infty} v_k(l)= v^*(l),\\
    		&\lim_{k \to \infty} P(l,v_k)= P(l,v^*),\\
    		&\lim_{k \to \infty} E_k= E^*,
    	\end{aligned}
    	\right.
    	\label{eq:proofA_7}
    \end{equation}
	where $v^*(l)$ represents the optimal magnitude of tangential velocity command at position $l$, ${P}(l,v^*)$ represents the optimal energy consumption at position $l$, $E^*$ represents the optimal energy consumption.

\section{Proof of Corollary 1.1}
{
	For the error form \eqref{eq:velocity-error}, according to eq. \eqref{eq:pace controller}, perform the k-th iteration projection in the tangential direction
	\begin{equation}
		\begin{aligned}
			{v}_k(l) = {v}^*_k(l)+\delta {v}_k(l).
		\end{aligned}
		\label{eq:proofB_error}
	\end{equation}
	With Corollary 1, there exists
	\begin{equation}
		\begin{aligned}
			\frac{\delta {v}_{k+1}(l)-\delta {v}_k(l)}{{v}_{k+1}(l)-{v}_k(l)} \leq  \mu,
		\end{aligned}
		\label{eq:proofB_1}
	\end{equation}
	where $ \mu < 1$ is a constant. With \eqref{eq:learning law}, \eqref{eq:nonlinear activation function}, \eqref{eq:proofA_3}, \eqref{eq:proofA_second derivative} and \eqref{eq:proofB_1}, we obtain
	\begin{equation}
		\begin{aligned}
			&P(l,v_{k+1}) \leq P(l,v_k) + \frac{\partial P(l,v)}{\partial v}\bigg|_{v=v_k}(v_{k+1}-v_k)\\
			&=P(l,v_k) + \frac{k_{\chi}}{ 1-\mu }\left(\frac{Mk_{\chi}}{ 2(1-\mu) }-1\right)\left(\frac{\partial P(l,v)}{\partial v}\bigg|_{v=v_k}\right)^2,
		\end{aligned}
		\label{eq:proofB_2}
	\end{equation}
	when $ \frac{Mk_{\chi}}{ 2(1-\mu) }-1 < 0$, namely $ k_{\chi} < \frac{2(1-\mu )}{M}$, there always has 
	\begin{equation}
		\begin{aligned}
			\frac{k_{\chi}}{ 1-\mu }\left(\frac{Mk_{\chi}}{ 2(1-\mu) }-1\right)\left(\frac{\partial P(l,v)}{\partial v}\bigg|_{v=v_k}\right)^2 \leq  0.
		\end{aligned}
		\label{eq:proofB_5}
	\end{equation}
	Therefore, under $k_{\chi}<2(1-\mu)/M$, the same monotone-convergence argument as in Appendix~A applies, and the limiting relations in\eqref{eq:proofA_7} hold. 
}

\nocite{1}
\bibliographystyle{IEEEtran}
\bibliography{IEEEabrv,IEEEbib}

\end{document}